\documentclass[twoside,11pt]{article}

\usepackage{jmlr2e}

\usepackage{amsmath}
\usepackage{hyperref}
\usepackage{array}
\usepackage{mathrsfs}
\usepackage{enumerate}
\usepackage{caption}
\usepackage{subfigure}
\usepackage{paralist}

\usepackage{thmtools}
\usepackage{thm-restate}

\newcommand{\recoverproperty}{locally symmetric actions}

\ewrlheading{14}{2018}{October 2018, Lille, France}{Yao Liu and Emma Brunskill}
\ShortHeadings{When Simple Exploration is Sample Efficient}{Liu and Brunskill}

\begin{document}

\title{When Simple Exploration is Sample Efficient: Identifying Sufficient Conditions for Random Exploration to Yield PAC RL Algorithms}

\author{\name Yao Liu 
\email yaoliu@stanford.edu \\
\addr Stanford University 
\AND
Emma Brunskill 
\email ebrun@cs.stanford.edu\\
\addr Stanford University 
}

%\author{Anonymous Authors}
%\editor{}

\maketitle

\begin{abstract}
Efficient exploration is one of the key challenges for reinforcement learning (RL) algorithms. Most traditional sample efficiency bounds require strategic exploration. Recently many deep RL algorithms with simple heuristic exploration strategies that have few formal guarantees, achieve surprising success in many domains. These results pose an important question about understanding these exploration strategies such as $e$-greedy, as well as understanding what characterize the difficulty of exploration in MDPs. In this work we propose problem specific sample complexity bounds of $Q$ learning with random walk exploration that rely on several structural properties. We also link our theoretical results to some empirical benchmark domains, to illustrate if our bound gives polynomial sample complexity in these domains and how that is related with the empirical performance.
\end{abstract}


\section{Introduction}
An important challenge for reinforcement learning is to balance exploration and exploitation. There have been many strategic exploration algorithms \citep{auer2007logarithmic,strehl2012incremental,dann2015sample}, yet many of the recent successes in deep reinforcement learning rely on algorithms with simple exploration mechanisms. While some of these approaches also require many samples, this still highlights an important question: when is exploration easy? In particular, we consider when a simple approach of random exploration followed by greedy exploitation can enable a strong efficiency criteria, Probably Approximately Correct (PAC): that on all but a number of sample that scales as a polynomial function of the domain, the algorithm will take near-optimal actions. Random exploration followed by greedy exploitation approach is related to popular $e$-greedy methods: it can be viewed as a particular thresholding decay schedule in $e$-greedy methods: $e$ is initially set to 1, and then dropped to 0 after a fixed number of steps. This simplification enables us to focus on when random exploration can still be efficient, and there are many domains where having a fixed budget for exploration is reasonable where our analysis will directly apply. Most prior work on formal analysis of exploration before exploitation approach \citep{langford2008epoch,kearns2002near} focused on strategic exploration during the exploration phase. In contrast, to our knowledge our work is the first to consider under what conditions random action selection during the exploration phase might still be sufficient to enable provably sample efficient reinforcement learning.

Some restrictions on the decision process are needed: there exist challenging Markov decision processes where relying on random exploration will require an exponential bound (in the MDP parameters) on the sample complexity, in contrast to the polynomial dependence required for the algorithm to be PAC. In some such domains, like the combination lock setting\citep{li2012sample,whitehead2014complexity}), any greedy actions will (for a very long time) cause the agent to undo productive exploration towards finding the optimal policy, and therefore $\epsilon$-greedy (for any $\epsilon$) will be no better and likely worse than random exploration, and therefore will also not have PAC performance. 

Rather than focusing on new algorithmic contributions, in this paper we seek to explore sufficient conditions on the domains that ensure that random exploration then exploitation methods will quickly lead to high performance, as formalized by satisfying the PAC criteria. Our work is related to recent work \citep{jiang2016structural} which considered structural properties of Markov decision processes that bound the loss when performing shallow \textit{planning}: in contrast to their work, our work focused on the structural properties of MDPs that enable simple exploration to quickly enable good performance during \textit{learning}. 

%In general there exist challenging MDPs where random exploration or $e$-greedy will require an exponential bound on the sample complexity, such as the combination lock setting \citep{li2012sample,whitehead2014complexity}. Rather than focusing on new algorithmic contributions, in this paper we seek to explore sufficient conditions on the domains that ensure that random exploration then exploitation methods will quickly lead to high performance. 

As our main contribution, we introduce new structural properties of MDPs, and prove that when these parameters scales with a polynomial function of the domain parameters, then a random explore then exploit approach is PAC. Our key properties are $\phi(s)$, a state’s stationary occupancy distribution under random walk, and eigenvalues of a graph Laplacian. Though making an assumption of the occupancy distribution under a random walk might seem to be presuming the conclusion, we note that this assumption only applies to the asymptotic, stationary distribution but our result yields finite sample bounds. Our result relies on some key results about convergence of a lazy random walk on directed graph in \citet{chung2005laplacians}. We also show that if a domain exhibits a property we term \textit{\recoverproperty} then it immediately satisfies the desired stationary criteria. That basically means for any two states there is a symmetric bijection between actions leading to the other state. A number of common simulation domains or slight variants of, including grid worlds, 4 rooms, and Taxi, satisfy this criteria. Following from this property, our work also yields some insights into why certain popular Atari domains have been observed to be feasible with simple e-greedy exploration. Some conditions that are known to enable efficient exploration under more strategic exploration algorithms, such as finite diameter domains, are not sufficient for a random exploration then exploit algorithm to be PAC, and we frame a classic domain, chain, as such an example. Our results also illustrate the difficulty of other similar ``trapdoor'' domains, including Montezuma’s Revenge which has been notoriously challenging for many deep RL agents. We also discuss several other properties that have been proposed to help characterize the learning complexity of MDPs and their relation to our proposed criteria.

To summarize, our results help to characterize the properties of an environment that make exploration hard or easy, a critical problem in RL. We hope these properties might help guide practitioners in their algorithm selection, and also advance our understanding about whether and when strategic exploration is needed. 

%\emma{I think we should frame early how this is related and distinct from prior work. I've moved the related work up, but it now needs to be remolded given it will come before our results}
\section{Related Works}
The optimality of the greedy policy in various settings has been previously studied for significantly more restricted settings. \citet{bastani2017exploiting} prove that a greedy policy can achieve the optimal asymptotic regret for a two-armed contextual bandit, which could be viewed as a special case of episodic reinforcement learning, as long as the contexts are i.i.d. and the distribution of contexts are diverse enough. That implies a case in contextual bandit where the greedy strategy is enough to solve the exploration problem. \citet{karush1967optimal} shows that under MDP structures, a greedy strategy is optimal, eliminating the need to plan ahead. Our work focuses on the random walk side of explore-greedy and yields a polynomial sample complexity bound under more mild assumptions.

Similarly, if the $Q$-functions are initialized extremely optimistically, $O(\frac{V_{\max}}{\Pi_{i=1}^T(1-\alpha_i)})$ where $\alpha_i$ is learning rate and $T$ is the samples we need to learn a near optimal $Q$ function, then greedy-only $Q$-learning is PAC \citep{even2002convergence}.
However, such a high optimism value (far higher than the possible achieve value) will result in an extremely aggressive exploration, further amplifying the problem of theoretically-motivated optimistic approaches in practice.

\citet{maillard2014hard} propose a notion of hardness for MDPs named as \textit{environmental norm}. It measures how varied the value function is at the possible next states and on the distribution over next states. They show how this property provides a tighter regret bound for UCRL algorithm. In the settings we consider random walk exploration is not driven by any reward/value observation, but purely depends on transition dynamics. Thus in this work we mainly consider transition-only parameters. In addition, in contrast to their work, we are focused on how structural properties of the MDP enable explore-greedy to be efficient, rather than improving the analysis of strategic exploration algorithms. 

Our proposed properties, stationary distribution and Laplacian eigenvalues, are related to a couple of other domain properties that have been previously considered. The first is diameter. Finite diameter is assumed for several strategic exploration algorithms such as optimism under uncertainty approaches \citep{jaksch2010near} and PAC analysis \citep{brunskill2013sample}. However, in the context of simple random exploration, a diameter that is polynomial with the MDP parameters is necessary but not sufficient. This is illustrated later in our chain example in which the diameter is finite, because there does exist a policy that could traverse between the start and end state in time linear in the state space, but under random walk the number of samples needed to be likely to reach a later state scales exponentially with later states. Our bound use stationary distribution to measure the asymptotic occupancy instead of direct reachability, which is measured by diameter. The second is proto-value functions \citet{mahadevan2007proto}, which use spectral properties of MDP to design a representation-based policy learning algorithm. Mixing time for MDPs is also a property that is closely related with stationary distribution and our bounds. Previous work about mixing time in MDPs \citep{kearns2002near, brafman2002r} aims at designing strategic exploration algorithm and bounding the complexity of it by mixing time. Mixing time for MDPs \citep{kearns2002near} is a property that is closely related with our bound. Previous work about mixing time in MDPs \cite{kearns2002near, brafman2002r} aim at designing strategic exploration algorithm and bounding the complexity of it by mixing time. Our bound focus on how the simple exploration method works, and we bound this variant of mixing time by other basic parameters as well as stationary distribution and eigenvalues. Our work is also related to classic results about cover time in Markov chains. Some bounds \citep{levin2017markov, ding2011cover} on $\epsilon$-mixing time and relaxation time can also induce a bound on cover time by stationary distribution and Laplacian eigenvalues, but they all focus on reversible chains, which our Theorem \ref{thm:laplacian} does not need.

%\emma{I think we need to say whether this is a trivial or nontrivial extension. I'm not sure if this paragraph should come here or immediately after our thm-- what do you think?}

%% === Section: Problem Settings === 
\section{Preliminaries}
An MDP is a tuple $M=\{\mathcal{S},\mathcal{A},P,R,\gamma\}$, where $\mathcal{S}$ is the state space, $\mathcal{A}$ is the action space, $P:\mathcal{S} \times \mathcal{A} \times \mathcal{S} \mapsto [0,1]$ is the probabilistic transition function, and $R: \mathcal{S} \times \mathcal{A} \mapsto [0,R_{\max}]$ is the reward function. We use $S$ and $A$ to denote the size of $\mathcal{S}$ and $\mathcal{A}$. The value $V^\pi(s)$ defines a discounted expected reward of running policy $\pi$ beginning with state $s$. Sample complexity \citep{kakade2003sample}, a way to quantify the performance of a reinforcement-learning algorithm, is defined as the total number of steps where algorithm execute a sub-optimal policy \textit{i.e.} $V^{*}(s) - V^\pi(s) > \epsilon$. An algorithm is PAC-MDP if its sample complexity is bounded by a polynomial function about $S$, $A$, $\frac{1}{\epsilon}$, $\frac{1}{\delta}$, and $\frac{1}{1-\gamma}$ with high probability.

Previous work \citep{even2003learning} that studies the polynomial convergence time of $Q$ learning by viewing exploration strategy as a black box. They characterize the efficiency of exploration by covering length and bound the convergence time by it.
\begin{definition}
The covering length, denoted by $L$, is the number of time steps we need to visit all state-action pairs at least once with probability at least $1/2$, starting from any $(s,a)$.
\end{definition}
\begin{theorem}
\label{thm:qlearning}
(Theorem 4 from \citet{even2003learning}) Let $Q_T$ be the value function after $T$ step $Q$ learning update, with learning rate $\alpha_t(s,a) = 1/(\#(s,a))^\omega$. $L$ is the covering length of the exploration policy. Then with probability at least $1 - \delta$, $\|Q_T - Q^{*}\|_\infty \le \epsilon$ if:
\begin{equation*}
T \ge T_0 = \widetilde{\Theta} \left( \left( L^{1+3\omega}V_{\max}^2 /((1-\gamma)\epsilon)^2\right)^{\frac{1}{\omega}}  + \left( L/(1-\gamma) \right)^{\frac{1}{1-\omega}} \right),
\end{equation*}
\end{theorem}
This theorem implies that, if the covering length $L$ of the exploration policy is polynomial in all parameters, we could learn the near optimal $Q$ function in polynomial time, and then achieve a near optimal policy by taking the greedy policy of this $Q$ function. Thus the covering length would be a good measure for us to evaluate the exploration quality of a policy, and it allows us to focus on exploration. In this work we consider Q learning combined with random walk exploration policy. We are interested in the minimum number of steps we need before switching to near-optimal greedy exploitation to guarantee a sufficient exploration. We intend to get a problem-specific bound by structural parameters of an MDP, to characterize when the exploration problem of an MDP is simple.

%% === Section: Bounds and Theories ===
\section{Covering Length Bound}
\label{sec:theory}
In this section, we will bound the covering length by the stationary distribution over states for random walk and Laplacian eigenvalues. The stationary distribution characterizes the asymptotic occupancy of states, and reflects asymptotically how good exploration will be. The smallest non-trivial eigenvalue of the Laplacian, is bounded by a geometric property named the Cheeger constant that intuitively measures the bottleneck of stationary random walk flow. These two parameters are both related to the asymptotic behavior of random walk. One natural question is that if we are given that asymptotically random walk can explore well, can we achieve polynomial sample complexity bound for finite sample exploration, and we show that through the following theorem.

Given the random walk policy $\pi_{RW}$, we have a transition matrix under this policy, $P^\pi_{RW}$, and we can view it as a transition matrix for a directed weighted graph, denoted as $G(P^\pi_{RW})$. If $P^\pi_{RW}(u,v)>0$ we say there is an edge from $u$ to $v$ with weight $P^\pi_{RW}(u,v)$ in $G$. For the rest of this section, we use $G$ and $P$ to refer to this graph and its transition matrix. It is known that for the transition matrix $P$, there is a unique left eigenvector $\phi$ such that $\phi(s)>0$ for any $s$ and $\phi P = \phi$, $\| \phi \|_1 = 1$. This eigenvector $\phi$ is also the stationary state distribution under the random walk policy. We follow the definition of graph Laplacian for a directed graph $G$ proposed by \citet{chung2005laplacians}:  
$$ \mathcal{L} = I -\frac{\Phi^{1/2}P\Phi^{-1/2} + \Phi^{-1/2}P^{*}\Phi^{1/2} }{2},$$
where $\Phi$ is a diagonal matrix with entries $\Phi(s,s) = \phi(s)$. Usually the graph Laplacian is only defined on undirected graph, and the intuition in \citep{chung2005laplacians} is that take the average of transition matrix $P$ and its transpose to define an undirected graph, then normalized the transition matrix, to introduce the Laplacian for weighted directed graph. The smallest eigenvalue of Laplacian $\mathcal{L}$ is zero. Let $\lambda$ be the smallest non-zero eigenvalue. In the following theorem, we will bound the covering time of random walk policy by the eigenvalues of $\mathcal{L}$ and the stationary distribution $\phi$.

\begin{theorem}
\label{thm:laplacian}
The covering length of a irreducible MDP under random walk policy is at most
$$8A\ln(4SA) \left( 2\ln\left(2/\min_s \phi(s)\right)/\ln(\frac{2}{2-\lambda}) + 1 \right)  \sum_{s}\frac{1}{\phi(s)},$$
where $\phi$ is the stationary distribution vector of random walk and $\lambda$ is the smallest non-zero eigenvalue of the Laplacian of the directed graph induced by random walk over MDP. The Laplacian is defined by \citet{chung2005laplacians}:  
$ \mathcal{L} = I -\frac{\Phi^{1/2}P\Phi^{-1/2} + \Phi^{-1/2}P^{*}\Phi^{1/2} }{2},$
where $\Phi$ is a diagonal matrix with entries $\Phi(s,s) = \phi(s)$ and P is the transition matrix $P(s,s') = \sum_{a} \frac{1}{A}T(s'|s,a)$.
\end{theorem}

It is known that in reversible Markov chains mixing time can be bounded by $\frac{1}{\epsilon\min_s\phi(s)} \frac{1}{1-\lambda^*}$ \citep{levin2017markov}, where $\lambda^*$ is the largest absolute value of eigenvalue of $P$, except $1$. Note that this $\lambda$ is the second largest eigenvalue of $P$ instead of the Laplacian, which is a normalized version of $I-P$, thus the relationshio between $\lambda^*$ and the second smallest eigenvalue of Laplacian, which is used in our paper, can be bounded. This mixing time bound gives us a cover time bound which has the same order of magnitude with our Theorem \ref{thm:laplacian}, in terms of $S$, $\lambda$ and $\min_s \phi(s)$. \citet{ding2011cover} also shows a similar result. Theorem \ref{thm:laplacian} remove the reversible assumption by considering the lazy random walk in directed graph and linking it to the cover time.

This bound immediately implies a PAC RL bound if $\frac{1}{\lambda}$ and $\frac{1}{\min_s \phi(s)}$ is polynomial.
\footnote{if $\phi_{\min}$ = 0 then this will be infinite, but this only occurs if the MDP is reducible. In that case, only the strongly connected component we are in is really matters for our exploration.} %Some states with exponentially small probability to be reach would hurt the bound, but that case is always hard since we do not know if there is hugh reward there or not, before we actually visit it.}
This shows that the Laplacian eigenvalue $\lambda$ and the stationary distribution are important factors for exploration. It is still not clear for what kind of MDPs these terms are polynomial. We will show two bounds for $1/\lambda$ and $1/\phi_{\min}$, which may provide more intuitive insight.

\textbf{Eigenvalue $\lambda$}: In graph theory, the second smallest eigenvalue of the Laplacian could be bounded by the Cheeger constant (also known as conductance). This will give us a more intuitive and geometric view of what $\lambda$ actually means for an MDP and when it is small. We define a flow over the graph induced by the stationary distribution of random walk as: $F(u,v) = \phi(u)P(u,v)$.
Then we write: $F(\partial U) = \sum_{u \in U, v \notin U} F(u,v)$, and $F(U) = \sum_{u \in U} \phi(u)$. 
The Cheeger constant is:
$h = \inf_{U} \frac{F(\partial U)}{\min \{ F(U), F(\overline{U})\}}$.
The Cheeger constant measures the relatively smallest bottleneck in the flow induced by stationary distribution. The Cheeger bound of $\lambda$ says that $h \ge \lambda \ge \frac{h^2}{2}$, which means $1/\lambda$ is polynomial if and only if $1/h$ is polynomial.

\textbf{Stationary distribution}: We know that for an (weighted) undirected graph, the stationary distribution on state $s$ is $O(\frac{d(s)}{\sum_s d(s)})$ where $d(s)$ is the degree of $s$. Then we define a property of MDPs:
\begin{definition}
An MDP has \recoverproperty~ if for any $s,s'$, there is a bijections $f$ between action sets $\{ a | P(s'|s,a)>0\}$ and $\{ a' | P(s|s',a')>0\}$ s.t. $P(s'|s,a) = P(s|s',f(a))$.
\end{definition}
If a MDP has \recoverproperty, we can construct an undirected graph such that the random walk on the MDP is equivalent with a random walk on this graph. The weight between two state in this graph is defined as:
$ w(u,v) = \sum_{a \in \mathcal{A}} P(v|u,a) = \sum_{a \in A} P(u|v,a) $.
One can verify that the random walk over the MDP has the same transition probability with random walk on this graph. Thus they also have the same stationary distribution, which is polynomial of $S,A$, computed from the undirected graph.

\begin{figure}[h!]
\centering
\subfigure{\label{fig:2roomoriginal}
\includegraphics[width=0.2\textwidth]{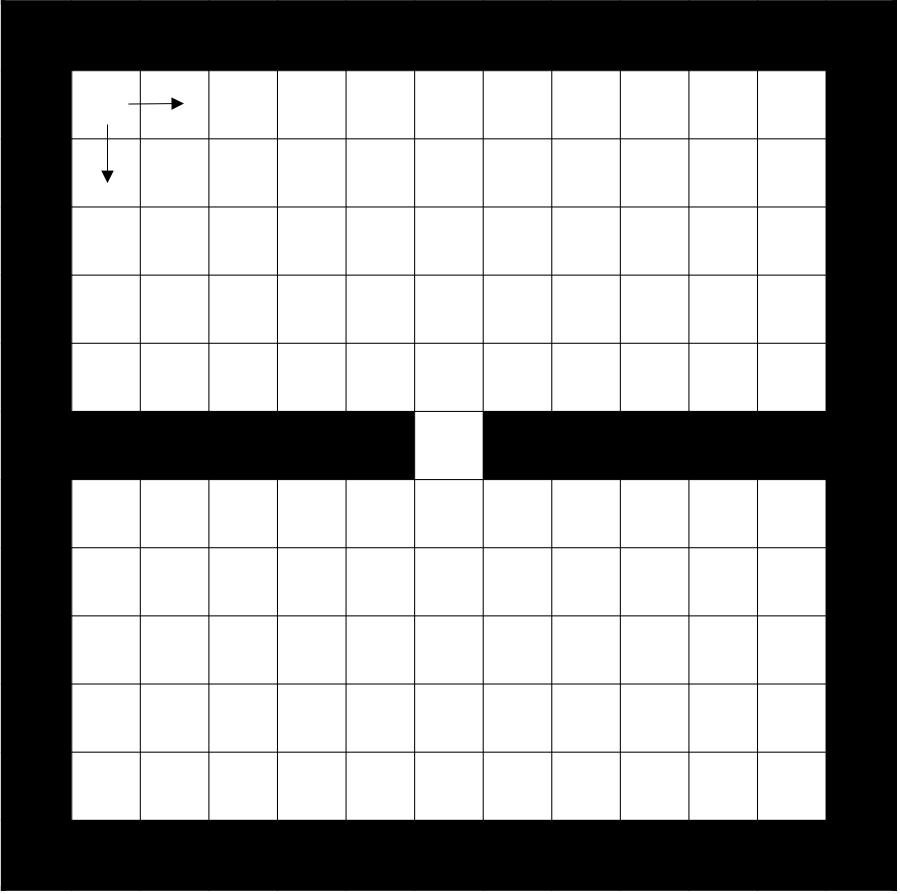}}
\subfigure{
  \includegraphics[width=0.3\textwidth]{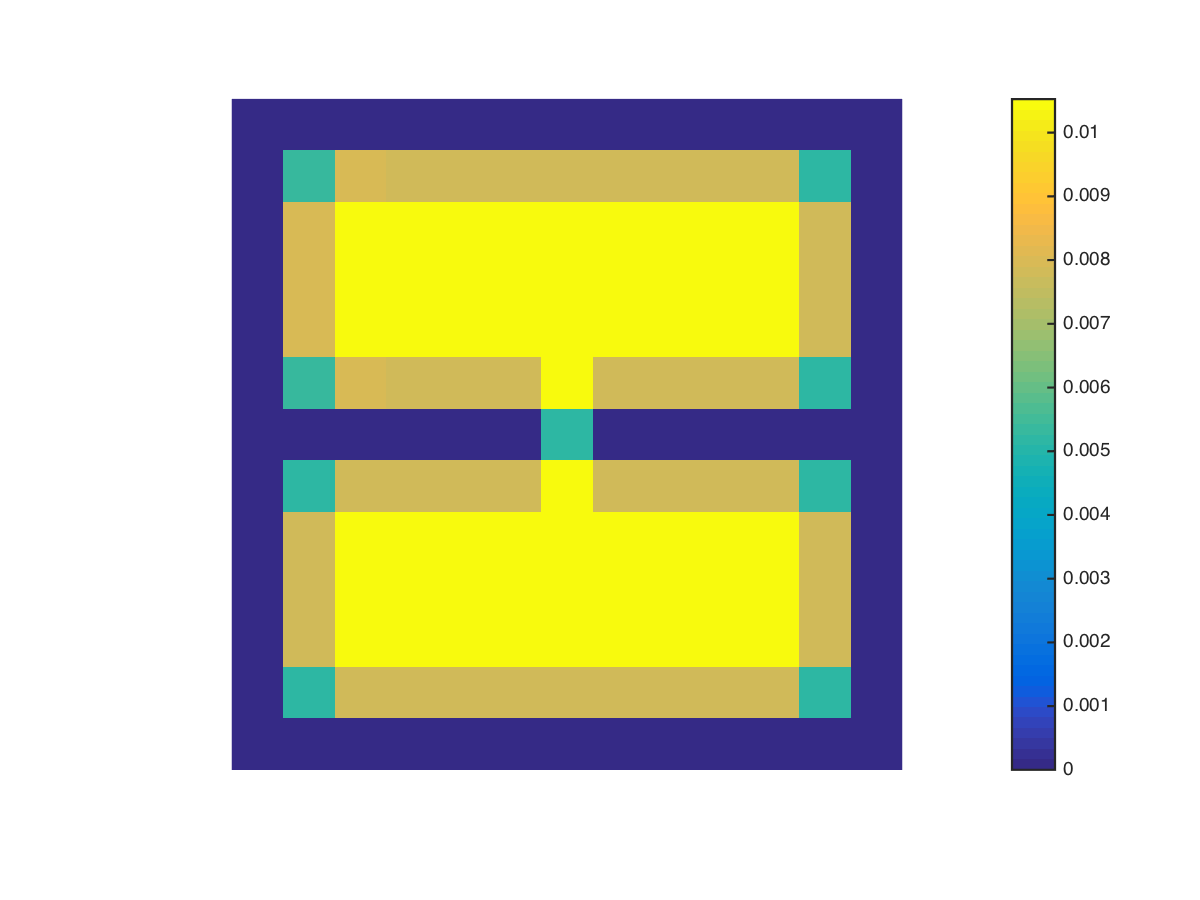}
   %\caption{Heat map of stationary distribution}
   \label{fig:2roomheat}}
%\caption{}
\caption{Left: two room domain. Right: the stationary distribution heat map}
\label{fig:2room}
\end{figure}

As a complementary, in appendix we list properties that give PAC RL bounds in certain cases where exploration should be easy intuitively, but are not covered by the bound in this section: When the actions behave similarly, or when all states are densely connected.

\section{Theoretical Bounds and Links to Empirical Results}%Connection with Practical Domains}
Our investigation was inspired by the recent empirical successes of deep reinforcement learning which relied on simple exploration mechanisms, and we hope that our theoretical analysis will both predict the hardness of domains that have been specifically constructed to require strategic exploration, as well add further insight into the hardness of other domains. In this section, we illustrate how our approach can explain some of the ease of exploration in some popular domains, as well as the hardness of exploration in others. %First we will show three domains that satisfy the property of \recoverproperty~.% so that $1/\phi_{\min}$ is polynomial with $S,A$ in these domains. %For these domain, we can also verify that the Cheeger constant is also polynomial.

\textbf{Grid World}: Grid world is a group of navigation domains where we need to control an agent to walk in a grid world, collect reward, avoid walls and holes. Most grid worlds with deterministic or other typical action settings have locally symmetric actions. Under this condition, random walk over the grid world is equivalent to a random walk on an undirected graph. Thus $1/\phi_{\min} = O(SA)$ and it is a polynomial function of MDP parameters.

\textbf{Taxi} \citep{dietterich2000hierarchical}: Taxi is a 5x5 gridworld. A passenger starts at one of the 4 locations marked in a grid world, and its destination is randomly chosen from one of the 4 locations. The taxi starts randomly on any square, and the goal is to pickup or dropoff the passenger. This domain, as well as the two room example we discussed previously, are widely used testing domains in the hierarchical RL literature, since options/modular policy are expected to achieve more efficient exploration than primitive actions. It is also equivalent with undirected graphs following from the property of \recoverproperty, if picking up/dropping off are not invertible actions. In that case, our bounds implies that random walk could learn the optimal value function of these domains efficiently. 

\textbf{Pong}: Pong is one of the Atari games that is relatively easy for DQN with $e$-greedy \citep{mnih2013playing}. In this domain, one plays pong with a computer player by moving the padder in $y$ axis, hitting the ball back. Interestingly, we can approximately view Pong as satisfying the property of \recoverproperty~by considering a state abstraction. In Pong, the angle of reflection is a bijection function of the hitting position on the paddle, not of angle of incidence, which implies that we could achieve any possible reflection angle in the possible angle domain by proper action. Consider a game state abstraction that consists only of the last ball incidence angle $\theta$ to the agent's paddle. That means, we view all frames after the ball leaves paddle until another hitting as the same state. This makes several notable simplifications, ignoring: the ball's velocity, boundary\footnote{Actually the boundary case could be treated by mirror reflection transformation. We would view the whole game as a mirror version of playing in the extended space, after hitting the boundary.}. Since we are playing in a boundless field, it is reasonable to view balls with different $y$ coordinates of hitting position as the same state. For simplicity we also assume the agent's opponent executes a deterministic policy that only depends on incidence angle, so that the mapping from incidence angle to reflection angle is a bijection, denoted as $f$. 

Under these settings, we can show Pong has the \recoverproperty. For any state $\theta_1$, if we execute an action $a_1$ so that the reflection angle is $\theta_1^{'}$, then the next state, which is the angle after the computer opponent takes an action would be $\theta_2 = f(\theta_1^{'})$. For this state, there exist an action $a_2$ such that the reflection angle is $f^{-1}(\theta_1)$. Since the mapping from action to reflection angle and $f$ are both bijection, the mapping between $a_1$ and $a_2$ is also bijection. Thus we could say random walk in a proper abstracted state space of Pong is equivalent with random walk on an undirected graph, and then yields polynomial sample complexity. That may intuitively explain the success of $e$-greedy in this domain.

%Our results also are consistent with some MDPs that are known to require strategic exploration. %In a $n$-state Chain MDP \citep{li2012sample,whitehead2014complexity}, it takes $\Theta(2^n)$ samples in expectation to visit the right end state for one time. That results in an exponential sample complexity. This matches our prediction: the stationary distribution  of random walk on state $s_{i}$ is $\Theta(\frac{1}{2^i})$. $1/\phi_{\min}$ is $\Theta(2^S)$, increasing exponentially with the state size in this case.

\textbf{Chain MDP}: The chain MDP has been previously introduced to motivated the need for strategic exploration\citep{li2012sample,whitehead2014complexity}. The MDP has n+1 states, the start state is the leftmost state $s_0$, and at each state $s_i$ there are 2 deterministic actions, one is going right to $s_{i+1}$ (except the right end states $s_{n}$ which has a self loop action) and the other is going back to $s_0$. $Q$ learning with $e$-greedy or random walk does poorly in this example. It takes $\Theta(2^n)$ samples in expectation to visit the right end state for one time, resulting in an exponential sample complexity. That matches what we can learn from our bound: The stationary distribution  of random walk on state $s_{i}$ is $\Theta(\frac{1}{2^i})$, and $1/\phi_{\min}$ is $\Theta(2^S)$. %, increasing exponentially with the state size in this case. %It can be easily verify that the inverse Cheeger constant is also $\Theta(2^S)$, by taking $U=\{s_{n}\}$. The bottleneck of stationary random walk flow is between $s_{n}$ and all other states, which implies that the difficulty of reaching $s_n$ in this example. 

\textbf{Montezuma's Revenge}: Montezuma's Revenge is a relatively hard game among different Atari 2600 games for DQN with $e$-greedy exploration\citep{mnih2013playing}. This game requires the player to navigate the explorer through several rooms. The explorer may die on the way of traps are triggered. We note that Montezuma's Revenge has a mechanism which brings one back to the start point after death. At a high level, that ``trapdoor'' structure is captured by the chain MDP example, and will result in an exponentially small stationary distribution of the end point. Game domains, even at a high level, may have more than one chain, but $\phi_{\min}$ could still be exponential in the maximum chain length. Note that some games like Pong or Enduro also have the restart mechanism, but that restart point is distributed more uniformly over the whole state space. This breaks the chain property and will not result in an exponentially small stationary distribution.

\begin{figure}[t!]
\centering
  \subfigure[Chain MDP]{
  \begin{minipage}[c]{0.2\textwidth}
  \centering
  \label{fig:chain}
  \includegraphics[width=\textwidth]{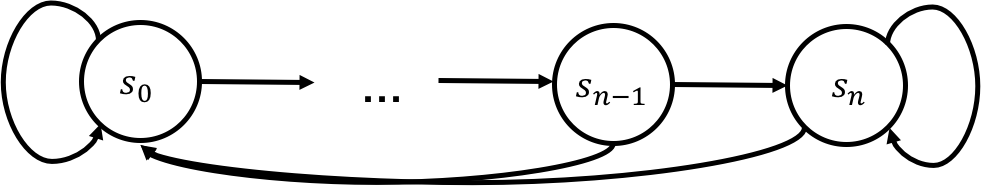}
  \end{minipage}
  %\caption{Chain MDP}
  }
  \hspace{0.5in} 
  \subfigure[Grid World]{
  \centering
  \includegraphics[width=0.2\textwidth]{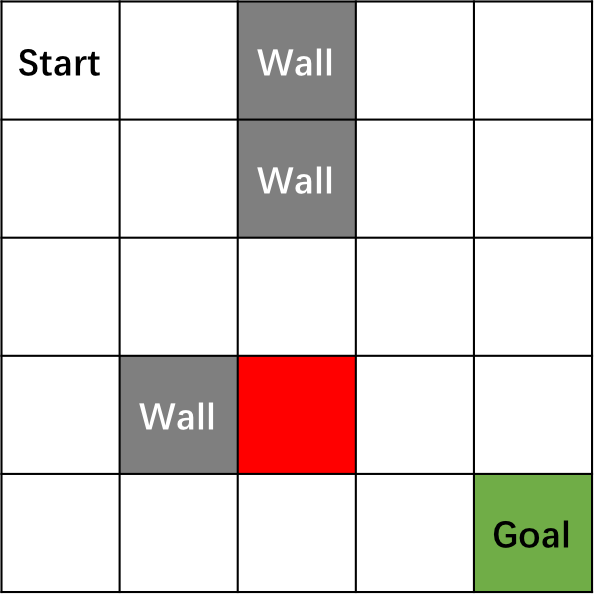}
  \label{fig:grid}
  }
  \hspace{0.5in} 
  \subfigure[Taxi \citep{dietterich2000hierarchical}]{
  \centering
  \includegraphics[width=0.2\textwidth]{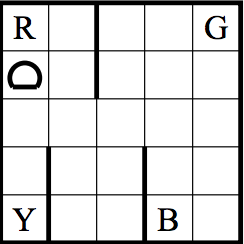}
  \label{fig:taxi}
  }
\caption{Domains with different order of stationary distribution}
%\label{fig:test}
\end{figure}

\section{Conclusion}
In this paper we present several structural properties of MDPs that give upper bound on the sample complexity of $Q$ learning with random exploration followed by exploitation. We also link these properties to some conceptual testing domains as well as empirical benchmark domains, towards understanding the recent empirical success. We hope the knowledge of these properties might help guide practitioners in selecting exploration strategy, and understanding whether and when strategic exploration is necessary. 

%Extending our results to generic $e$-greedy algorithms is also an interesting future direction. There are some reasons for hope. For example, if optimal actions do not require long horizon planning, such as in Pong, then $e$-greedy is likely to still be PAC, because greedy actions will not unduly influence the agent's ability to effectively explore. The work from \citet{jiang2016structural} that bound loss due to short planning horizon may relate this setting to more general cases. In other cases, we may put assumptions over reward structure or intrinsic reward settings to bound the performance of greedy actions. In contrast, for adversarial reward structure, e-greedy may not be PAC in general, since greedy actions could be easily distracted by local rewards and make it hard to explore enough.

% There is more work to be done towards finding problem specific bound for linear $e$-greedy instead of random walk exploration, which we leave for future explorations. To bound the performance of greedy part, we may rely on additional assumptions of reward structure or intrinsic reward settings. 

\acks{
This work was supported in part by Siemens and a NSF CAREER grant.}

\bibliography{ref.bib}

\newpage
\setcounter{theorem}{0}

\appendix
\section{Preliminaries}
For completeness and clarity we include some definitions and lemmas which is helpful in our proof, and included in the main body.
\begin{definition}
The covering length, denoted by $L$, is the number of time steps we need to visit all state-action pairs at least once with probability at least $1/2$, starting from any pair.
\end{definition}

\begin{theorem}
%\label{thm:qlearning}
(Theorem 4 from \citet{even2003learning}) Let $Q_T$ be the value function after $T$ step $Q$ learning update, with learning rate $\alpha_t(s,a) = 1/(\#(s,a))^\omega$. Then with probability at least $1 - \delta$, we have $\|Q_T - Q^{*}\|_\infty \le \epsilon$, given that
$$
T \ge T_0 = \Theta \left( \left( \frac{L^{1+3\omega}V_{\max}^2 \ln\left( \frac{SAV_{\max}}{\delta(1-\gamma)\epsilon}\right) }{(1-\gamma)^2\epsilon^2}\right)^{\frac{1}{\omega}} + \left( \frac{L}{1-\gamma} \ln\frac{V_{\max}}{\epsilon} \right)^{\frac{1}{1-\omega}} \right),
$$
where $L$ is the covering length of the exploration policy we use in $Q$ learning
\end{theorem}

\setcounter{theorem}{4}
Diameter \citep{auer2007logarithmic} is a widely used parameter to measure the reachability of the MDP. Intuitively it means the longest expected time to reach one state from the other. More formally:
\begin{definition}
(Diameter)
$$ D = \max_{s,s'} \min_{\pi} \mathbb{E} \left[ \inf\left\{ t \in \mathbb{N}: s_t = s' \right\} | s_0 = s, \pi \right] $$
\end{definition}

The following lemma allows us to only focus on how to cover all states in the later analysis.
\begin{lemma}
\label{lem:coveract}
If we visit a state more than $A\ln(4SA)$ times, a random walk policy will sample every action at least once with probability at least $1-\frac{1}{4S}$.
\end{lemma}

For completeness, we also include a lemma about relation of $Q$ value accuracy and its greedy policy performance, which is widely used in $Q$ learning literature.
%Once we get near optimal $Q$ function, the following lemma guarantees that the greedy policy is near optimal.
\begin{lemma}
\label{lem:optimalpi}
Let $\pi$ be the greedy policy of an action value function $Q$. If $\|Q^*-Q \|_{\infty} \le \epsilon$, then $\|V^{\pi^*} - V^{\pi} \|_{\infty} \le \frac{2\epsilon}{1-\gamma}$.
\end{lemma}

\section{Proofs in Section \ref{sec:theory}}
In this section, we include the full proofs of the three main theorems and lemmas in the main body of paper. For the completeness and convenience of reading, we also include the lemmas and proofs that are stated in the main body. 
\subsection{Laplacian Eigenvalues and Stationary Distribution}

To prove theorem \ref{thm:laplacian}, we introduce a useful lemma from \citet{chung2006diameter} which bounds the convergence of lazy random walk (random walk with additional 0.5 probability that will stay in the same state) over a directed graph $G$. Then we will relate the lazy random walk transition matrix with the one-way commute time of random walk over $G$.

\begin{lemma}
\label{lem:lazywalk}
Suppose a strongly connected weighted directed graph has transition matrix $P$, and a lazy random walk transition $\mathcal{P} = \frac{(I+P)}{2}$. For any state $u,v$ and $k>0$, the normalized matrix $M = \Phi^{1/2}\mathcal{P}\Phi^{-1/2}$ satisfies:
$$ \left| M^{k}(u,v) - \sqrt{\phi(u)\phi(v)} \right| \le (1-\lambda/2)^{k/2} $$
\end{lemma}
This is part of the result in the Theorem 1 from \citet{chung2006diameter}.

\begin{corollary}
\label{col:lazywalk}
$ \mathcal{P}^k(u,v) \ge  \phi(v) - \sqrt{\frac{\phi(v)}{\phi(u)}}(1-\lambda/2)^{k/2}$
\end{corollary}

\begin{proof}
We have that
$$ M^{k}(u,v) \ge \sqrt{\phi(u)\phi(v)} - (1-\lambda/2)^{k/2} $$
from lemma \ref{lem:lazywalk}. Since $M^k = \Phi^{1/2}\mathcal{P}^k\Phi^{-1/2}$. Then 
$$ \mathcal{P}^{k}(u,v) = \frac{1}{\sqrt{\phi(u)}}M^{k}(u,v) \sqrt{\phi(v)} \ge  \phi(v) - \sqrt{\frac{\phi(v)}{\phi(u)}}(1-\lambda/2)^{k/2} $$
%$$ d_{G}(u,v) \le \left\lfloor \frac{2\ln(\sqrt{\phi(u)\phi(v)})}{\ln(\frac{2}{2-\lambda})}\right\rfloor + 1 $$
\end{proof}
Now we could bound $\mathcal{P}^{k}(u,v)$ by the graph Laplacian properties. The next lemma shows that $\mathcal{P}^{k}(u,v)$ is a lower bound of the probability of reaching $v$ from $u$ under random walk over $G$.

\begin{lemma}
\label{lem:walkprob}
Suppose a strongly connected weighted directed graph has transition matrix $P$, and a lazy random walk transition $\mathcal{P} = \frac{(I+P)}{2}$. The the probability of reaching $v$ from $u$ within $k$ steps by original random walk will be at least $\mathcal{P}^k(u,v)$.
\end{lemma}
\begin{proof}
For simplicity of discussion, we firstly assume that $P_{i,i} = 0$ for any $i$, which means there is no self loop in the original random walk. At the end of proof, we will show that how this proof still works for the case with self loop.

Define $F(u,v;k)$ as the probability of reaching $v$ from $u$ within $k$ steps by original random walk. Let $l = (s_0 = u, s_1, ... , s_t = v)$ be a path from $u$ to $v$ with length $0 < t \le k$, and for all $i<t$, $s_i \neq v$. We call this kind of path \textit{first-visit} path. Then we could compute $F(u,v;k)$ by sum the probability over all first-visit path. Let $\mathcal{L}_{uv}$ be the set of all first-visit paths from $u$ to $v$ with length $0 < t \le k$.
\begin{equation*}
F(u,v;k) = \sum_{l \in \mathcal{L}_{uv}} Pr(l|\text{r.w.}) = \sum_{l \in \mathcal{L}_{uv}} \prod_{i=0}^{t-1}P(s_i,s_{i+1}), 
\end{equation*}
where the sum is over all distinct first-visit path with length less than $k$.

Note that $\mathcal{P}^k(u,v)$ is the probability of reaching $u$ from $v$ at $k$th step by lazy random walk. Let $\mathcal{L}$ be the set of paths with length of $k$ in the lazy random walk graph, whose transition weight matrix is $\mathcal{P}$ but not $P$. Then
$$
\mathcal{P}^k(u,v) = \sum_{\widehat{l}\in \mathcal{L}}Pr(\widehat{l}|\text{lazy r.w.})
= \sum_{\widehat{l}\in\mathcal{L}} \prod_{i=0}^{k-1}\mathcal{P}(\widehat{s}_i,\widehat{s}_{i+1}) 
$$

$\widehat{l}$ in $\mathcal{L}$ may not be a first-visit path, since there are lazy steps as well as extra steps after first visit. Now we will divide $\widehat{l}$ into three disjoint part and extract the first-visit part in $\widehat{l}$. Firstly we find the first visit of $v$ in $\widehat{l}$, and let $\widehat{l}_{uv}$ be all the steps in $\widehat{l}$ from $u$ to the first visit to $v$ without all lazy steps. Since the lazy steps are self loop, $\widehat{l}_{uv}$ is still a valid path. Let the length of $\widehat{l}_{uv}$ be $t \le k$, and the number of all lazy steps in path $\widehat{l}$ be $i(\widehat{l})$. Then the rest steps in $\widehat{l}$ is a path from $v$ to $v$ with length of $k-t-i$. Let this path be $\widehat{l}_{vv}$. Note that for all $\widehat{l}$, $\widehat{l}_{uv}$'s are first-visit paths with length no greater than $k$, and they cover all first-visit paths with length no greater than $k$. $\widehat{l}_{vv}$'s are a valid paths from $v$ to $v$ with length $k-t-i$, and they cover all paths from $v$ to $v$ with length $k-t-i$.

Now the problem is there might be more than one path $\widehat{l}$ with the same $\widehat{l}_{uv}$. We need to prove that $\mathcal{P}^k(u,v)$ does not count it more than one, which means for these $\widehat{l}$ with the same $\widehat{l}_{uv}$, $$\sum_{\widehat{l}}Pr(\widehat{l}|\text{lazy random walk}) \le Pr(\widehat{l}_{uv}|\text{random walk})$$ 
To prove it, let $\mathcal{L}(\widehat{l}_{uv})$ be the set of all $\widehat{l}$ with the same $\widehat{l}_{uv}$:
\begin{eqnarray*}
\sum_{\widehat{l} \in \mathcal{L}(\widehat{l}_{uv})}Pr(\widehat{l}|\text{lazy r.w.}) 
&=& \sum_{i=0}^{k-t} \sum_{\widehat{l} \text{ s.t. $i(\widehat{l})=i$}}Pr(\widehat{l}|\text{lazy r.w.}) \\
&=& \sum_{i=0}^{k-t} \sum_{\widehat{l} \text{ s.t. $i(\widehat{l})=i$}} \frac{1}{2^k} Pr(\widehat{l}_{uv}|\text{r.w.})  Pr(\widehat{l}_{vv}|\text{r.w.}) \\
&=& \sum_{i=0}^{k-t} Pr(\widehat{l}_{uv}|\text{r.w.}) \sum_{\widehat{l} \text{ s.t. $i(\widehat{l})=i$}} \frac{1}{2^k} Pr(\widehat{l}_{vv}|\text{r.w.}) \\
&=& \sum_{i=0}^{k-t} \frac{1}{2^k} \binom{k}{i} Pr(\widehat{l}_{uv}|\text{r.w.}) \sum_{\widehat{l}_{vv}, \, |\widehat{l}_{vv}|=k-t-i} Pr(\widehat{l}_{vv}|\text{r.w.}) \\
&=&\sum_{i=0}^{k-t} \frac{1}{2^k} \binom{k}{i} Pr(\widehat{l}_{uv}|\text{r.w.}) P^{k-t-i}(v,v) \\
&\le&\sum_{i=0}^{k-t} \frac{1}{2^k} \binom{k}{i} Pr(\widehat{l}_{uv}|\text{r.w.}) \\
&\le& Pr(\widehat{l}_{vv}|\text{r.w.})
\end{eqnarray*} 
By dividing of $\widehat{l}$ according to the value of $i$, we have the first steps. The second step follows from dividing the path $\widehat{l}$ into three parts: $\widehat{l}_{uv}, \widehat{l}_{vv}$, and the self-loop part. Note that for a step $(s,s')$ in $\widehat{l}$, $\mathcal{P}(s,s') = \frac{1}{2}$ for lazy self-loop steps and $\mathcal{P}(s,s')=\frac{P(s,s')}{2}$ for the other steps. The third step follows from that $Pr(\widehat{l}_{uv}|\text{r.w.})$ is a constant since $\widehat{l}_{uv}$ is fixed. For a fixed $i$ and fixed $\widehat{l}_{vv}$, there is $\binom{k}{i}$ different $\widehat{l}$, since there is $\binom{k}{i}$ possible combinations of lazy steps. By taking the some over these lazy steps combinations with a fixed $\widehat{l}_{vv}$, we have the fourth step. The fifth step follows from the fact that if we take sum of probability over all possible $k-t-i$ steps path from $v$ to $v$, then that is the probability of visiting $v$ from $v$ at $k-t-i$ steps. Since it is a valid probability, it is no greater than 1 and yields the sixth step. By substituting the result above into the expression of $F(u,v;k)$, we have that:
\begin{eqnarray*}
\mathcal{P}^k(u,v) &=& \sum_{\widehat{l} \in \mathcal{L}}Pr(\widehat{l}|\text{lazy r.w.}) = \sum_{\widehat{l}_{uv} \in \mathcal{L}_{uv}} \sum_{\widehat{l} \in \mathcal{L}(\widehat{l}_{uv})}Pr(\widehat{l}|\text{lazy r.w.}) \le \sum_{\widehat{l}_{uv} \in \mathcal{L}_{uv}} Pr(\widehat{l}_{uv}|\text{r.w.})
\end{eqnarray*}
The last line is exactly $F(u,v;k)$, completing the proof.

Now consider the case that there exist self loops in the original transition matrix $P$. In that case, we can split the self loops in $P$ from the self loops in $I$. For example, if there is a path in lazy random walk $\widehat{l} = (\widehat{s}_0, \dots, \widehat{s}_i = s, \widehat{s}_{i+1} = s, \widehat{s}_k )$. In the original path $Pr(\widehat{s}_{i+1} = s|\widehat{s}_{i}=s) = (P(s,s)+1)/2$. We can split this path into two exactly same path: $\widehat{l}_1$ and $\widehat{l}_{2}$. In $\widehat{l}_1$, $Pr(\widehat{s}_{i+1} = s|\widehat{s}_{i}=s) = P(s,s)/2$, and this transition step is part of the sub-path $\widehat{l}_{uv}$. In $\widehat{l}_2$, $Pr(\widehat{s}_{i+1} = s|\widehat{s}_{i}=s) = 1/2$, and this transition step is part of the lazy steps. This decomposition does not change the probability under lazy random walk since $Pr(\widehat{l}|\text{lazy r.w.}) = Pr(\widehat{l}_1|\text{lazy r.w.}) + Pr(\widehat{l}_2|\text{lazy r.w.})$. Thus the analysis for no self loop case works for $\widehat{l}_1$ and $\widehat{l}_2$, and we finish the proof for all transition matrix $P$ cases.
%$= (\widehat{s}_0, \dots, \widehat{s}_i = s, \widehat{s}_{i+1} = s, \widehat{s}_k )$
\end{proof}
Combining this result with corollary \ref{col:lazywalk}, we immediately have the following result:
\begin{corollary}
\label{col:commuteprob}
For any two states $u$, $v$, the probability of reaching $v$ from $u$ within $k$ steps is at least $\phi(v) - \sqrt{\frac{\phi(v)}{\phi(u)}}(1-\lambda/2)^{k/2}$.
\end{corollary}
By setting the time steps $k$ large enough, we could lower bound the one-way commute probability by the stationary distribution:

\begin{corollary}
\label{col:commuteprobk}
For any two state $u$, $v$, the probability of reaching $v$ from $u$ within $k$ steps is at least $\phi(v)/2$, for any $k \ge k_0 = \frac{2\ln\left(2/\phi_{\min}\right)}{\ln(\frac{2}{2-\lambda})} + 1$, where $\phi_{\min}$ is $\min_{x \in \mathcal{S}} \phi(x)$.
\end{corollary}
\begin{proof}
By substitute $k$ with $\left\lfloor \frac{2\ln\left(2/\sqrt{\phi(u)\phi(v)}\right)}{\ln\left(\frac{2}{2-\lambda}\right)} \right\rfloor $ in corollary \ref{col:commuteprob}, we have that the probability is bounded by $\phi(v)/2$. Since $\sqrt{\phi(u)\phi(v)}>\phi_{\min}$, $k \ge k_0 \ge \frac{2\ln\left(2/\sqrt{\phi(u)\phi(v)}\right)}{\ln\left(\frac{2}{2-\lambda}\right)}$.
\end{proof}
Now we need a high probability bound for the one way commute time $k$ between two states.
\begin{corollary}
\label{col:commutetime}
For any two state $u$, $v$, we can visit $v$ from $u$ at least $A\ln(4SA)$ time with probability $1-\frac{1}{4S}$, within $\frac{8A\ln(4SA)k_0}{\phi(v)}$ steps.
\end{corollary}
\begin{proof}
We know that for $k_0 = \frac{2\ln\left(2/\phi_{\min}\right)}{\ln(\frac{2}{2-\lambda})} + 1$ steps, we can visit $v$ with probability at least $\phi(v)/2$. This is a Bernoulli trial with success probability at least $\phi(v)/2$. Note that for different $u$, they are all Bernoulli trials with a same lower bound of success probability and $k$. By lemma 56 in \cite{li2009unifying}, if we do $\frac{4(A\ln(4SA)+\ln4S)}{\phi(v)}$ trials, we will have A ln (4SA) successes with probability at least $1-1/4S$. We can do such number of trials by no more than $\frac{8A\ln(4SA)k_0}{\phi(v)}$ time steps.
\end{proof}

\setcounter{theorem}{2}
\begin{theorem}
(Restated)
%\label{thm:laplacian}
The covering length of a irreducible MDP under random walk policy is at most
$$8A\ln(4SA) \left( \frac{2\ln\left(2/\min_s \phi(s)\right)}{\ln(\frac{2}{2-\lambda})} + 1 \right)  \sum_{s}\frac{1}{\phi(s)},$$
where $\phi$ is the stationary distribution vector of random walk and $\lambda$ is the smallest non-zero eigenvalue of the Laplacian of graph induced by random walk.
\end{theorem}
\setcounter{theorem}{13}
\begin{proof}
Firstly, by combining corollary \ref{col:commutetime} and lemma \ref{lem:coveract}, we have that with probability $1- 1/2S$, we can visit every action in state $v$ within $\frac{8A\ln(4SA)k_0}{\phi(v)}$, starting from any state. Applying this for every state $v$, we have that with probability at least $1/2$, we can cover every state action pair within $8A\ln(4SA)k_0\sum_{s}\frac{1}{\phi(s)}$ steps.
\end{proof}

This bound immediately implies a sufficient condition of a PAC RL bound as the next corollary states.
\begin{corollary}
\label{col:LaplPACcondition}
For any irreducible MDP M, let $\mathcal{L}$ be the Laplacian of the graph induced by random walk over M, $\lambda$ be the smallest non-zero eigenvalue of $\mathcal{L}$, and $\phi(s)$ be the stationary distribution over states by random walk. If:
\begin{compactenum}
\item $\frac{1}{\lambda}$ is a polynomial function of the MDP parameters, and
\item $\frac{1}{\min_s\phi(s)}$ is a polynomial function of the MDP parameters,
\end{compactenum}
then $Q$ learning with random walk exploration is a PAC RL algorithm.
\end{corollary}
\begin{proof}
Since $1-\frac{1}{x} \le \ln(x)$, we have that $$ \frac{1}{\ln(\frac{2}{2-\lambda})} = \frac{1}{\ln(\frac{1}{1-\lambda/2})} \le \frac{1}{1-(1-\lambda/2)} = \frac{2}{\lambda}$$ 
Since $\frac{1}{\lambda}$ and $\frac{1}{\min_s\phi(s)}$ is polynomial with MDP parameters, we have that $L$, as well as $T$ in theorem \ref{thm:qlearning} are also polynomial. Thus we achieve near optimal policy after polynomial number of mistakes if we switch to greedy policy of the learned $Q$ function after $T$ steps.
\end{proof}

\section{Other Structural Properties that Bound Covering Length}
In the proceeding sections, we have looked at problem specific bounds for exploration that depends on stationary distribution and Laplacian eigenvalue. Yet, there are MDPs that are easy to explore but not covered by this bound. While covering all these cases is beyond the objective of this work, we cover two classes of MDPs where exploration is intuitively easy.

\subsection{Action Variation}
One natural class of MDPs that exploration is easy for random walk are those where different actions at the same state have similar distribution over the next states. In that case, random walk could easily cover all the next states and may result in a very similar behavior with the best exploration policy. We capture this class of MDPs by the property action variation, which was introduced by \citet{jiang2016structural} to bound the loss of shallow planning.
\begin{definition} (Action Variation)
\footnote{It is slightly different with the action variation defined by \citet{jiang2016structural}. Their definition of action variation consider the maximum $l_1$ distance between two actions' transition vectors.}
$$ \delta_P = \max_{s} \max_{a} \left\| P(\cdot|s,a) - \frac{1}{A}\sum_{a'}P(\cdot|s,a') \right\|_1 $$
\end{definition}

We need to introduce some useful lemmas before we prove the main theorem in this section. Firstly we will define commonality between two probability distribution and a elementary fact of commonality, then include a key lemma from \cite{jiang2016structural} for completeness.

% \begin{definition}
% $$ \delta_P = \max_{s} \max_{a} \left\| P(\cdot|s,a) - \frac{1}{A}\sum_{a'}P(\cdot|s,a') \right\| $$
% \end{definition}

\begin{definition}
Given two vectors p, q of the same dimension, define comm(p,q) as the commonality vector of p and q, with entries comm(s;p,q) = $\min \{p(s),q(s)\}$.
\end{definition}

\begin{proposition}
\label{prop:comm}
$$ \| \mathrm{comm}(p,q) \|_1 = 1-\| p-q\|_1/2 $$
\end{proposition}

\begin{proposition} (lemma 1 in \cite{jiang2016structural})
\label{prop:comm_norm}
For any stochastic vector p, q and transition matrix $P_1$, $P_2$
$$ \| \mathrm{comm}(p^TP_1,q^T P_2) \|_1 \ge  \| \mathrm{comm}(\mathrm{comm}(p,q)^T P_1,\mathrm{comm}(p,q)^T P_2) \|_1 $$
\end{proposition}

We also need the next helping lemma which is widely used in MDP approximation analysis:
\begin{proposition}
\label{prop:holder} (lemma 2 in \cite{jiang2016structural})
Given stochastic vectors $p,q$, and a real vector $v$ with the same dimension,
$ | p^Tv - q^Tv| \le \| p-q\|_1 \max_{s,s'}|v(s) - v(s')|/2 $
\end{proposition}

\begin{lemma}
\label{lem:actvar}
Let p and q be two stochastic vectors over $S$, $\pi$ be any policy and $\pi_{RW}$ be the random walk policy. Then
$$
\| \mathrm{comm}(p^T P^{\pi}, q^T P^{\pi_{RW}}) \|_1 \ge (1-\delta_P/2)\|\mathrm{comm}(p,q) \|_1 
$$
\end{lemma}
\begin{proof}
	\begin{eqnarray}
\| \mathrm{comm}(p^T P^{\pi}, q^T P^{\pi_{RW}}) \|_1 &\ge& \| \mathrm{comm}(\mathrm{comm}(p,q)^T P^{\pi},\mathrm{comm}(p,q)^T P^{\pi_{RW}}) \|_1 \\
&=&  \| \mathrm{comm}(p,q) \|_1 \| \mathrm{comm}(z^T P^{\pi},z^T P^{\pi_{RW}}) \|_1 \\
&=& \| \mathrm{comm}(p,q) \|_1 ( 1- \| z^T (P^{\pi}- P^{\pi_{RW}}) \|_1/2) \\
&\ge&  \| \mathrm{comm}(p,q) \|_1 ( 1- \delta_P/2)
\end{eqnarray}
The first step use proposition \ref{prop:comm_norm}. $z$ is a normalized vector of $\mathrm{comm}(p,q)$. So the second step follows from scaling. The third step follows proposition \ref{prop:comm}. Note that $l_1$ norm each row of $P^{\pi}- P^{\pi_{RW}}$ is bounded by $\delta$. The last step follows from the fact that $l_1$ norm is a convex function.
\end{proof}

The following theorem bounds the covering length in the case that either the actions have almost identical transition or the diameter is small, which implies that the necessary planning horizon is short.
\begin{theorem}
\label{thm:actvar}
For an MDP with finite diameter $D$, if $\delta_P \le \frac{2}{5D}$, then the covering length $L = O\left(DSA\ln(SA)\right)$. Thus the Q learning with random walk exploration could learn the near optimal Q function within polynomial steps.
\end{theorem}
\begin{proof}
Now consider a target MDP with respect to a particular state $s$, where the transition is as same as the original MDP, but state $s$ is the absorbing state and has the only unit reward. By Markov inequality and definition of diameter, the optimal policy can visit $s$ with in $cD$ steps with probability at least $(c-1)/c$ in the original MDP. Since the target MDP has the same transition with the original MDP except the state $s$, the expectation visiting time of $s$ would not change. So the undiscounted value of optimal policy in target MDP would be at least $(c-1)dD/c$ for $(c+d)D$ steps. Now let us compute the undiscounted $T$ steps value for random walk policy. Let $p$ be the distribution vector of start state, r be the reward distribution vector. (Note that the reward we defined for target MDP only depends on state.)
\begin{equation}
V^{\pi^*} - V^{\pi_{RW}} = \sum_{k=0}^{T} p^T (P^{\pi^*})^k r - \sum_{k=0}^{T} p^T (P^{\pi_{RW}})^k r = \sum_{k=0}^{T} (p^T (P^{\pi^*})^k - p^T (P^{\pi_{RW}})^k)r
\end{equation}
By using lemma \ref{lem:actvar} $k$ times, we have that 
\begin{equation}
\mathrm{comm}(p^T (P^{\pi^*})^k, p^T (P^{\pi_{RW}})^k) \ge (1-\delta/2)^k\mathrm{comm}(p,p)  = (1-\delta/2)^k 
\end{equation}
Use proposition \ref{prop:comm} to turn commonality into $l_1$ error:
\begin{equation}
\| p^T (P^{\pi^*})^k - p^T (P^{\pi_{RW}})^k \|_1 \le 2-2(1-\delta/2)^k
\end{equation}
Substitute this into the value error above:
\begin{eqnarray}
|V^{\pi^*} - V^{\pi_{RW}}| &\le& \sum_{k=0}^{T} |(p^T (P^{\pi^*})^k - p^T (P^{\pi_{RW}})^k)r| \\
&\le& \sum_{k=0}^{T} \|(p^T (P^{\pi^*})^k - p^T (P^{\pi_{RW}})^k)\|_1 \max_{s,s'}|r(s) - r(s^{'})|/2 \\
&\le& \sum_{k=0}^{T} (1-(1-\delta_P/2)^k) R_{\mathrm{max}}
\end{eqnarray}
So the value of $\pi_{RW}$ could be bounded by 
\begin{equation}
V^{\pi^*} - T + \frac{1-(1-\delta_P/2)^T}{1-(1-\delta_P/2)} \ge \frac{(c-1)dD}{c} - (c+d)D + \frac{2}{\delta_P}(1-(1-\delta_P/2)^T) 
\end{equation} 
% Let $c=2$, $d=2$, and simplify the equation above we have that $V^{\pi_{RW}} \ge D(1-4D\delta)$. On the other hand, the escape probability of random walk in $T = 4D$ steps is:
%Let $c = d = \frac{1}{2D\delta_P}$.
If $T\delta_P/2 \le 1$, then by Taylor extension we have:
\begin{eqnarray*} 
%(1-\delta_P/2)^{T} = (1-1/4cD)^{2CD} < \lim_{x \to \infty}(1-1/2x)^x = 1/\sqrt{e} 
(1-\delta_P/2)^{T} &=& 1 - \frac{T\delta_P}{2} + \frac{T(T-1)}{2}\left(\frac{\delta_P}{2}\right)^2 + \sum_{k=3}^{\infty} \frac{(-1)^k T!}{k!(T-k)!}\left(\frac{\delta_P}{2}\right)^k \\
&\le& 1 - \frac{T\delta_P}{2} + \frac{T(T-1)}{2}\left(\frac{\delta_P}{2}\right)^2 - \frac{T(T-1)(T-2)}{6}\left(\frac{\delta_P}{2}\right)^3 \\ && + \sum_{k=4}^{\infty} \frac{T!}{k!(T-k)!}\left(\frac{\delta_P}{2}\right)^k \\
&\le & 1 - \frac{T\delta_P}{2} + \frac{T(T-1)}{2}\left(\frac{\delta_P}{2}\right)^2 - \frac{T(T-1)(T-2)}{6}\left(\frac{\delta_P}{2}\right)^3 \left(1-\sum_{k=4}^\infty\frac{6}{k!} \right) \\
&\le& 1 - \frac{T\delta_P}{2} + \frac{T(T-1)}{2}\left(\frac{\delta_P}{2}\right)^2
\end{eqnarray*}
Thus
\begin{equation*}
V^{\pi^*} - T + \frac{1-(1-\delta_P/2)^T}{1-(1-\delta_P/2)} \ge T - \frac{T^2\delta_P}{4} - \left(c+\frac{d}{c}\right)D \ge \frac{3(c+d)D}{4}-\left(c+\frac{d}{c}\right)D
\end{equation*}
Let $c=2$ and $d=3$, the value above is $D/4$. We have that $ V^{\pi_{RW}} \ge D/4$. Remember that we also need $T\delta_P/2 \le 1$. Since we assume $\delta_P \le \frac{2}{5D}$ that is true for $T=5D$.
On the other hand, the probability of visit $s$ by random walk within $T$ steps is:
\begin{equation}
p_v = \sum_{k=0}^{T} Pr(\text{visit s at kth step}) \ge \sum_{k=0}^{T} Pr(\text{visit s at k step})\frac{T-k}{T} = \frac{V^{\pi_{RW}}}{T} \ge \frac{1}{20} 
\end{equation}
Every $T = 5D $ steps episode we have a constant probability to visit state $s$. Recall that at each state we uniformly draw actions. According to \ref{lem:actvar}, we need to visit a state more than $A\ln(4SA)$, so that with probability at least $1-1/4S$ we sample every action at least once. By lemma 56 in \cite{li2009unifying}, we can yield this by $O(A\ln(4SA) + \ln(4S))$ episodes with probability at least $1-1/4S$. Applying this for every state s and combine the fail probability, we have that with probability at least $1/2$, we can visit every state-action pair within $O\left(DSA\ln(SA)\right)$. That completes the proof.
\end{proof}

Note that this bound being polynomial does not imply that the stationary distribution is polynomial, since there are MDPs where actions are almost the same, but some certain states could only be achieved under exponentially small probability. Also it is obvious that the bound in \ref{thm:laplacian} is polynomial also does not imply the polynomial bound here.

There could be cases in RL applications that action variation is small. Note that action variation only measure the difference in transition dynamics, and the reward can still vary a lot in this case. In hierarchical RL domains, it is common that more than one options leads to the same goal, with different cost/reward. For example, if we want to control a robot arm to pick up a cup, there are many ways to pick up a cup that all end up with cup in the hand. Rewards can be very different here but the outcome space is the same.

\subsection{Sub Transition Matrix Norm}
Let us view an MDP from a graph perspective where actions are edges between states. If the graph is dense, then we can easily visit any states quickly, and intuitively we do not need to look ahead for too many steps to achieve a good exploration strategy. In that case, the MDP is easy to explore intrinsically and we want to get a problem specific bound for random walk exploration in this case. 

Let $P$ be the transition matrix under random walk $\pi_{RW}$, and $P_{-v,-v}$ be the sub-matrix of $P$ except column and row corresponding to the state $v$.
\begin{lemma}
\label{lem:commutebound}
For any state $v$, the one-way covering time from any state to $v$ by policy $\pi$ is bounded by:
$$  \max_{u} \mathbb{E}  \left\{ \inf\left\{ t \in \mathbb{N}: s_t = v \right\} | s_0 = u, \pi \right\} = ||(I-P^T_{-v,-v})^{-1} ||_1 $$
\end{lemma}
\begin{proof}
Let $e_u$ be the one hot start state vector with only entry on $u$, and this is a $S-1$ dimension vector since we remove the state $v$. Let $X$ be the random variable of the time we first visit $v$, then $Y=X-1$ would be the last time of we stay in $\mathcal{S}\slash {v}$. The probability of not visiting $v$ within $k$ steps is $\| e_u^T P_{-v,-v}^k \|_1$, which means:
$$ Pr(Y \ge k) = \| e_u^T P_{-v,-v}^k \|_1$$
Thus we could compute the expectation of $X$ by:
\begin{eqnarray}
\mathbb{E}(X) &=& \sum_{k=1}^\infty Pr(X\ge k) = \sum_{k=0}^\infty Pr(Y\ge k) \\
&=& \sum_{k=0}^{\infty} \| e_u^T P_{-v,-v}^k \|_1 = \| e_u^T \sum_{k=0}^{\infty}P_{-v,-v}^k \|_1 \\
&=& \|e_u^T(I-P_{-v,-v})^{-1} \|_1
\end{eqnarray}
The second line is true since elements in $e_u^T P_{-v,-v}^k$ is non-negative for all $k$. Note that $\max_{u}\|e_u^T(I-P_{-v,-v})^{-1} \|_1$ is exactly the $l_1$ norm of matrix $(I-P^T_{-v,-v})^{-1}$
\end{proof}

Thus, to bound the covering length under $\pi$ by this, we only need to bound $||(I-P^T_{-v,-v})^{-1} ||_1$. By prove the equivalence factor between matrix norm by Holder's inequality, we have the following result. 
\begin{lemma}
\label{lem:normbound}
If $\inf_p ||P^T_{-v,-v} ||_p < 1$,
$$ ||(I-P^T_{-v,-v})^{-1} ||_1 \le \inf_{p \in \mathbb{N}}\frac{S^{(1-1/p)}}{1-||P^T_{-v,-v} ||_p} $$
\end{lemma}
\begin{proof}
For any $n$-by-$n$ matrix $A$ and $p \ge 1$:
$$ \|A\|_1 = \max_{x} \frac{\| Ax\|_1}{\| x\|_1} \le \max_{x} \frac{n^{1-1/p}\| Ax\|_p}{\| x\|_1} \le \max_{x} \frac{n^{1-1/p}\| Ax\|_p}{\| x\|_p} = n^{1-1/p}\| A\|_p $$
The firstly inequality follows from Holder's inequality, and the second one is simply from $\| x\|_1 \ge \| x\|_p $ for any $p\ge 1$. For any matrix induced $l_p$ norm,
\begin{equation} 
||(I-P^T_{-v,-v})^{-1} ||_p \le \sum_{k=1}^{\infty} \| P^k_{-v,-v} \|_p \le \sum_{k=1}^{\infty} \| P_{-v,-v} \|_p^k = \frac{1}{1-||P^T_{-v,-v} ||_p}
\end{equation}
Now combine these together, we have that:
\begin{equation}
||(I-P^T_{-v,-v})^{-1} ||_1 \le \inf_{p \ge 1} \left[ (S-1)^{(1-1/p)}||(I-P^T_{-v,-v})^{-1} ||_p \right] =  \inf_{p \ge 1}\frac{S^{(1-1/p)}}{1-||P^T_{-v,-v} ||_p} 
\end{equation}
\end{proof}

Note that the bound is finite only if the sub transition matrix of policy $\pi$ satisfies $\inf_p ||P^T_{-v,-v} ||_p < 1$. By repeating this enough times, as bounded in lemma \ref{lem:coveract}, we have the upper bound of steps for covering all actions in state $i$. Applying this to every state, we can get the upper bound of covering length for random walk, as the following theorem:

\begin{theorem}
\label{thm:submatrixlp}
Let $P$ be the transition matrix under random walk policy $\pi_{RW}$, and $P_{-v,-v}$ be the sub-matrix of $P$ except column and row corresponding to $v$. If for any state $v$, $\inf_p ||P^T_{-v,-v} ||_p < 1$. The covering length of this MDP under random walk is finite and bounded by:
$$ 4A\ln(4SA) \sum_{v \in \mathcal{S}} \inf_{p \ge 1}\frac{S^{(1-1/p)}}{1-||P^T_{-v,-v} ||_p}  $$
\end{theorem}
\noindent\textbf{Remark}: The assumption $\inf_p ||P^T_{-v,-v} ||_p < 1$ is more likely to be true when the transition matrix $P$ is more dense. The following corollary will give us a intuition about this. If we only consider the case $p=1$ it will be reduced to a trivial bound:
\begin{corollary}
If the minimum one step transition probability between two different states under $\pi_{RW}$ is $p_{min} > 0$, then the covering length is bounded by 
$\frac{ 4SA\ln(4SA) }{p_{min}}$
\end{corollary}
\begin{proof}
This corollary immediately follows from the case $p=1$ in theorem above, and the fact that $ 1-||P^T_{-v,-v} ||_1 = p_{min}$.
\end{proof}

\end{document}